\documentclass[runningheads]{llncs}
\usepackage{lineno,hyperref}
\usepackage{array,multirow}
\usepackage{graphicx}
\usepackage{subfig}
\usepackage{mathrsfs}  
\usepackage{amsfonts}
\usepackage{amsmath}
\usepackage{lmodern}

\newcommand{\eqdef}{\overset{\mathrm{def}}{=\joinrel=}}

\makeatletter
\renewcommand\@biblabel[1]{}
\makeatother
\modulolinenumbers[5]

\usepackage{algorithm2e}

\begin{document}

\title{Feature Selection by a Mechanism Design\thanks{The views expressed herein are those of the author and should not be attributed to the IMF, its Executive Board, or its management.}}

\author{Xingwei Hu}
\authorrunning{X Hu}
\institute{International Monetary Fund\\ Washington, DC 20431, USA \\
	\email{xhu@imf.org}}

\maketitle              

\begin{abstract}
In constructing an econometric or statistical model, we pick relevant features or variables from many candidates.
A coalitional game is set up to study the selection problem where the players are the candidates and the payoff function is a performance measurement in all possible modeling scenarios.
Thus, in theory, an irrelevant feature is equivalent to a dummy player in the game, which contributes nothing to all modeling situations. 
The hypothesis test of zero mean contribution is the rule to decide a feature is irrelevant or not.
In our mechanism design, the end goal perfectly matches the expected model performance with the expected sum of individual marginal effects.
Within a class of non-informative likelihood among all modeling opportunities, the matching equation results in a specific valuation for each feature.
After estimating the valuation and its standard deviation, we drop any candidate feature if its valuation is not significantly different from zero.
In the simulation studies, our new approach significantly outperforms several popular methods used in practice, and its accuracy is robust to the choice of the payoff function. 

\end{abstract}

\section{Introduction}\label{sect:introduction}
 
Feature selection or variable selection is one of the few most fundamental goals in statistical learning (e.g., Zou, 2006).
In the selection process, the algorithm picks some relevant features or variables in constructing a model, dropping redundant or irrelevant ones.
In general, all these candidate features correlate with each other at various levels; some may be partially redundant. 
To address this issue, this paper designs a mechanism in a coalitional game and evaluates each candidate feature. 
Based on the evaluation profile, each candidate is selected or not by a statistical decision.

There are dozens of selection methods in the literature. 
They generally trade off the model fit with the model size (e.g., Fan and Li, 2006); that is,
they penalize the model fit statistic or loosen up the minimal discrepancy statistic in order to achieve a certain level of simplicity.
The model fit could also be a measure of the error for out-of-sample prediction.
However, truth and accuracy can't be completely compromised with simplicity if we pursue the true features or causality in the data generating process, rather than solely making the model smaller.

The trade-off solutions may not result in highly precise identification of the true features. 
First, a fit statistic or penalty measure could have many variants, each leading to a different solution; and none dominates the others. 
The BIC (Bayesian or Schwarz Information Criterion) method has a fixed penalty in terms of the sample size and the number of unknown parameters, for example, while the Lasso's penalty  is proportional to the magnitude of the unknown parameters (Tibshirani, 1996). 
Secondly, each trade-off may have a set of particular assumptions. 
For example, the results from the AIC (Akaike Information Criterion) and BIC are sensitive to the normality-distribution assumption (e.g., Dziak et al., 2020); 
subset search algorithms ignore the model uncertainty (e.g., Clyde and George, 2004);
Lasso and its adaptive version (Zou, 2006) assume the independence among the regressors.
Also, to reduce the model size by sacrificing the model fit, it is subtle to leverage the unknown amount of the penalty, and few researches have studied the delicacy of the balance.
Ridge regression, for example, shrinks all coefficients by a unknown uniform factor but does not set any coefficients to zero.
Calibration of the unknown factor, however, links back to the coefficients.
Lastly, minimization of the penalized discrepancy functions or objective functions generally reaches local optimal solutions.
In summary, these methods could estimate a good (``best" or ``optimal" in some sense) solution but not exclusively seek the true regressors.

In this paper, we propose a new approach which avoids the trade-off by ignoring any model fitness. 
We use the candidate features and a model performance function to set up a coalitional game.
The choice of the performance function is of secondary or even little importance as long as any irrelevant feature is statistically a dummy player in the game.
In a specific model, a feature's contribution is its marginal effect to the model.
Allowing for model uncertainty, we also assume a class of non-informative priors for the modeling scenarios.
So a feature's importance is the expected marginal effect in all potential modelings.
Thus, our approach is also Bayesian and uses model averaging (e.g., George and McCulloch, 1997;  O'Hara and Sillanpaa, 2009), except that the prior is not pre-set but identified from a deliberate design. 

By design, we match the expected performance of the potential model such that the value of our discrepancy objective function is zero on average, making any penalty unlikely. 
The matching also identifies a unique prior from the class and a specific valuation formula for all candidate features.
Using this particular prior, we estimate the value of each feature and its standard error.
We mark a candidate feature as irrelevant if its valuation is not statistically significant.
When comparing the new approach with others, such as AIC, BIC, stepwise regression, Lasso, and adaptive Lasso using simulation,
we find that it largely outperforms them.

The rest of the paper is organized as follows: 
Section \ref{sect:coalitional_game} sets up a coalitional game for the valuation problem. 
Section \ref{sect:valuation} derives the valuation formula by solving a matching equation. 
The section also justifies the valuation by the Shapley value (1953) and the dummy player property.
Section \ref{sect:est_algorthms} provides two algorithms to estimate the valuation and its standard error.
Section \ref{sect:compare_simulation} compares the simulation performance between our new solution with five payoff functions and five other variable selection methods.
Finally, Section \ref{sect:conclusion} concludes with further comments.
Our exposition is self-contained, and the proofs are in the Appendix.

\section{Coalitional Game with Random Coalitions}\label{sect:coalitional_game}
 
In the exposition, we use the following notations.
All candidate features $X_1$, $X_2$, ..., $X_n$ are collectively denoted by the set $\mathbb{N} = \{1, 2,\cdots,n\}$. 
For any $T\subseteq \mathbb{N}$, $v(T)$ is a performance measure when we model the data using the features or variables from $T$.
For the empty set $\emptyset$, $v(\emptyset)$ is the performance when the model does not involve any variables from $\mathbb{N}$.
As usual, ``$\setminus$" is for set subtraction and ``$\cup$" for set union. 
The overbar is used in naming the element of a singleton set; for example,  ``$\overline{i}$"  for the set $\{ i \}$.
For any subset $T\subseteq \mathbb{N}$, $|T|$ denotes its cardinality.
For brevity, we often use $t$ and $n$ for $|T|$ and $|\mathbb{N}|$, respectively.

For a simple example, we consider a linear regression for the dependent variable $Y_t$ where $t$ is the observation index for the data.
There are $n$ candidate explanatory variables $X_{it}$, $i=1,2,...,n$, in addition to some pre-set regressors $Z_{jt}$.
The general unrestricted model (GUM) is then
\begin{equation}\label{eq:GUM}
Y_t = \beta_0 + \sum\limits_{i\in \mathbb{N}} \beta_i X_{it} + \sum\limits_j \alpha_j Z_{jt} + \epsilon_t
\end{equation}
where $\beta_i$ and $\alpha_j$ are unknown coefficients and $\epsilon_t$ is the residual.
The performance measure $v$ could be the variance explained, the forecast accuracy, the significance statistic,
the probability of avoiding outliers, etc. 
In particular, $v(\emptyset)$ is the performance when $Y_t$ is modeled by the constant $\beta_0$ and $Z_{jt}$'s.

At the inception of the learning process, the machine may have no knowledge about the true features in the data.
To address the model uncertainty, we let the random subset $\mathbf{S}\subseteq \mathbb{N}$ consist of the features in the true model. 
We also denote the estimated model by $\hat{\mathbf{S}} \subseteq \mathbb{N}$.
In contract to the GUM model (\ref{eq:GUM}), therefore, the true model reduces to 
\begin{equation}\label{eq:true_model}
Y_t = \beta_0 + \sum\limits_{i \in \mathbf{S}} \beta_i  X_{it} + \sum\limits_j \alpha_j Z_{jt} + \epsilon_t
\end{equation}

\begin{figure}
\centering
\parbox{5cm}{
\centering
\includegraphics[height=3cm, width=3cm]{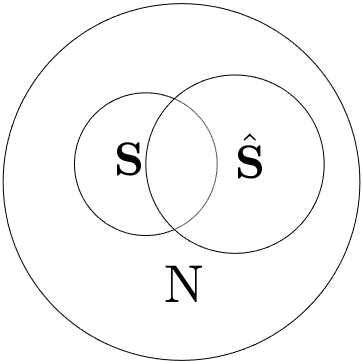}
\caption{Two types of errors.}
\label{fig:two_types_error}}
\qquad
\begin{minipage}{5cm}
\centering
\includegraphics[height=3cm, width=4.5cm]{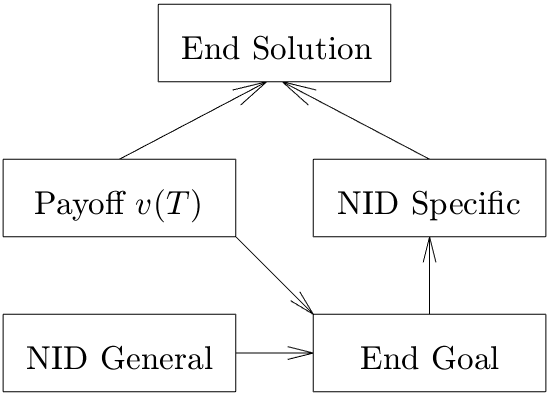}
\caption{The mechanism design.}
\label{fig:design}
\end{minipage}
\end{figure}

In general, $\hat{\mathbf{S}} \not = \mathbf{S}$ and the estimation could make two types of errors (cf Figure \ref{fig:two_types_error}).
In Type 1, called ``under-fitting", the selection process fails to select some relevant features, which are in $\mathbf{S} \setminus \hat{\mathbf{S}}$.
In the second type,  called ``over-fitting", it wrongly selects some irrelevant features, which are in the set of $\hat{\mathbf{S}} \setminus \mathbf{S}$.
Also, $\hat{\mathbf{S}}$ could have both errors when neither $\hat{\mathbf{S}} \setminus \mathbf{S} = \emptyset$ nor $\mathbf{S} \setminus \hat{\mathbf{S}} = \emptyset$.
Finally, when $\hat{\mathbf{S}}=\mathbf{S}$, there is neither under-fitting nor over-fitting; and the model is exactly identified.
We measure the precision of the estimation method by the percentage of exact identification in a large number of simulated models.

The design works on the uncertainty of the modeling scenarios.
For any $T\subseteq \mathbb{N}$, let $P_{_T}$ be the probability of $\mathbf{S} = T$. 
With no prior knowledge about the potential model, the learning process has no reason to discriminate any two subsets of $\mathbb{N}$ with the same size.
This defines a class of non-informative distributions (NID) for $\mathbf{S}$:
\vskip .1cm
\begin{center}
\textbf{NID} Assumption : $P_{_T} = P_{_Z}$ whenever $T$ and $Z$ have the same size.
\end{center}
\vskip .1cm
Unlike the other Bayesian variable selection methodologies, the NID assumption does not specify a unique distribution or a density over a family of distributions.
Rather, it is a restriction over the density of $\mathbf{S}$. 
Under the NID assumption, the next section solves a unique solution for $P_{_T}$ by taking an objective-first approach.
In the objective, the expected $v(\mathbf{S})$ is distributed to all candidate features by their expected marginal effects in $v$.
Thus, the design problem is the inverse of the traditional Bayesian analysis, which is typically devoted to the analysis of the posterior performance for a given prior.
Figure \ref{fig:design} illustrates the process. 
First, the payoff function $v$ and the NID restriction form the end objective of the design.
Next, the requirements of the end goal derive a updated NID density. 
Finally, the new NID probability and the function $v$ generate the final solution.

\section{Valuation of the Coalitional Game}\label{sect:valuation}

In the above section, we have three elements in the coalitional game: a player set $\mathbb{N}$; a payoff function $v: 2^\mathbb{N} \to R$; and a NID assumption on the random coalition $\mathbf{S}$.
In this section, we modify the efficiency axiom of the Shapley value (1953, page 309) to derive a unique solution for all players in $\mathbb{N}$.

Player $i$'s value in the game could be explained by its marginal effect in $v(T)$ when $\mathbf{S}=T$. 
When $i \not \in T$, then $T$ faces an opportunity cost
$v(T \cup \overline{i}) - v(T)$, due to $i$'s absence from $T$.
In other words, the feature could have increased the collective performance by $v(T\cup \overline{i} ) - v(T)$ if we had added it to $T$. 
This scenario happens with probability $P_{_T}$.
Alternately speaking, this marginal effect is equivalent to $v(T) - v(T \setminus \overline{i})$ when $i\in T$, but with probability $P_{_{T\setminus \overline{i}}}$, i.e., 
$X_i$'s participation in the model contributes $v(T) - v(T \setminus \overline{i})$ to the collective performance.

We define the loss function as the expected marginal effect 
\begin{equation}\label{eq:lambda}
\lambda_i[v] \
\eqdef \ \mathbb{E} \left [v({\mathbf S}\cup \overline{i}) - v({\mathbf S}) \right ]
= \sum\limits_{T \subseteq \mathbb{N}} P_{_T} \left [v(T\cup \overline{i}) - v(T) \right ].
\end{equation}
Clearly, the loss function satisfies both the symmetry and the aggregation axioms of the Shapley value (1953, page 309), under the NID assumption.
But, the efficiency axiom, that is, $\sum\limits_{i \in \mathbb{N}} \lambda_i[v] = v(\mathbb{N}) - v(\emptyset)$, is unlikely valid because the true model performance is 
$v(\mathbf{S}) - v(\emptyset)$, not $v(\mathbb{N}) - v(\emptyset)$.
In general, $\mathbb{N}$ contains unrelated and redundant features; thus there is noise in $v(\mathbb{N})$, which should be stretched or squeezed toward $v(\mathbf{S})$.
However, $v(\mathbf{S})$ is random and is not visible.
Therefore, we replace the efficiency axiom with the following matching equation:
\begin{equation}\label{eq:divide_expectation_lambda}
\sum\limits_{i\in \mathbb{N}} \lambda_i[v;\mu] \equiv  \mathbb{E} v(\mathbf{S}) - v(\emptyset)
\end{equation}
where ``$\equiv$" means that the functional equation holds for any $v: 2^\mathbb{N} \to R$.
On average, $v(\mathbb{N}) - \mathbb{E} v(\mathbf{S})$ is the over-fitting that the candidate features make and should be omitted in our identification process.
By crowding out the average noise to zero, (\ref{eq:divide_expectation_lambda}) sets an identity restriction on the probability density of $\mathbf{S}$.
By Theorem \ref{thm:divide_expectation_lambda}, the restriction leads to a unique solution for $P_{_T}$.

\begin{theorem}\label{thm:divide_expectation_lambda} 
Assume the NID. 
Then (\ref{eq:divide_expectation_lambda}) holds if and only if 
\begin{equation}\label{eq:SV-prior}
P_{_T} = \frac{t!(n-t)!}{(n+1)!}, \quad \forall \ T \subseteq \mathbb{N}.
\end{equation}
\end{theorem}

\begin{proof} 
See Supplemental Appendix A1. 
\end{proof}

With the density $P_{_T}$ in (\ref{eq:SV-prior}), the size of $\mathbf{S}$ has a uniform distribution over the integers $0,1,2,...,n$.
Denote the probability of $|\mathbf{S}|=t$ by
$$
\delta_t \eqdef \sum\limits_{T\subseteq \mathbb{N}: |T|=t} P_{_T}, \quad \forall \ t=0,1,2,...,n.
$$
Then (\ref{eq:SV-prior}) implies that
\begin{equation}\label{eq:delta_sv}
\delta_t = 
\left ( 
\begin{array}{c}
n \\
t
\end{array}
\right )
P_{_T}
= \frac{n!}{t!(n-t)!} \frac{t!(n-t)!}{(n+1)!}
= \frac{1}{n+1}.
\end{equation}
Also with (\ref{eq:SV-prior}), the loss function (\ref{eq:lambda}) simplifies to
\begin{equation}\label{eq:loss_SV_prior}
\lambda_i[v] = \sum\limits_{T \subseteq \mathbb{N}} \frac{t!(n-t)!}{(n+1)!} \left [v(T\cup \overline{i}) - v(T) \right ].
\end{equation}
In contrast to the Shapley value, i.e.,
\begin{equation}\label{eq:shapley_value}
\Psi_i[v] = \sum\limits_{T \subseteq \mathbb{N}} \frac{t!(n-t-1)!}{n!} \left [v(T\cup \overline{i}) - v(T) \right ],
\end{equation}
$\lambda_i[v]$ has smaller weights on the marginal effects $\left [v(T\cup \overline{i}) - v(T) \right ]$,
and the largest de-weighting $\frac{n-t}{n+1}$ occurs when $t$ is close to $n$, shrinking the over-fitting.

We could also derive (\ref{eq:loss_SV_prior}) from the Shapley value and the dummy player property.
A dummy player $j$ has zero marginal effects on all $T\subseteq \mathbb{N}$, i.e. $v(T\cup \overline{j})-v(T)=0$ for all $T\subseteq \mathbb{N}$.
In theory, an irrelevant feature is equivalent to a dummy player in the coalitional game;
in reality, however, the payoff function $v$ and its marginal effects are subject to the sampling errors.
To fill the gap in between, a statistical decision rule judges if the feature's average marginal effect is zero, or tests if its marginal effects are white noise.
In the game with $\mathbf{S}$ as the carriers (Shapley, 1953), the players in $\mathbb{N} \setminus \mathbf{S}$ are indeed dummies.
The payoff function in the new game is defined as
$$
v_{_{\mathbf{S}}} (T) \eqdef v(\mathbf{S} \cap T), \quad \mathrm{for}\ \mathrm{any} \ T\subseteq \mathbb{N}.
$$
Any irrelevant feature, which is not in $\mathbf{S}$, is a dummy to $v_{_\mathbf{S}}$; hence, it has zero Shapley value in $\Psi [v_{_\mathbf{S}}]$.
As the efficiency axiom holds in the new game, we could apply the vector $\Psi [v_{_\mathbf{S}}]$ to distribute $v(\mathbf{S}) - v(\emptyset)$ to all members in $\mathbb{N}$, in which dummies get nothing. So does any irrelevant feature, on average, as implied in Theorem \ref{thm:expected_SV}.

\begin{theorem}\label{thm:expected_SV} 
Given the NID assumption for $\mathbf{S}$, the expected Shapley value $\Psi [v_{_\mathbf{S}}]$ is the loss function $\lambda[v]$, i.e.,
$\mathbb{E} \Psi_i [v_{_\mathbf{S}}] \equiv \lambda_i[v]$ for all $i \in \mathbb{N}$,
if and only if (\ref{eq:SV-prior}).
\end{theorem}

\begin{proof}
See Supplemental Appendix A2. 
\end{proof}

\section{Estimation Algorithms}\label{sect:est_algorthms}
 
In this section, we randomize the sequential approach in Shapley (1953) to estimate $\lambda[v]$ and its standard error.
Consequently, the decision rule is as: the feature $X_i$ is irrelevant if $\lambda_i[v]$ is not significantly different from zero.

Simulation can often solve complex problems.
In general, $n$ is large in practice, making exact calculation of (\ref{eq:loss_SV_prior}) computationally costly.
In contrast to the simple expression for the mean, the variance of $v(\mathbf{S}\cup \overline{i}) - v(\mathbf{S})$ involves the squared $\lambda_i[v]$ and the mean squared marginal effects, making the calculation even worse.
Thus, we seek an alternate random sampling technique which also involves $\lambda_i[v]$ as the mean. 
Let $\Omega$ be the set of orderings of all candidate features in $\mathbb{N}$.
We randomly take an ordering $\tau$ from $\Omega$:
\begin{equation}\label{fig:sequence}
\tau: \quad \emptyset\ \longrightarrow \  i_1\ \longrightarrow \ \cdots\  \longrightarrow \ i\ \longrightarrow \ \cdots\ \longrightarrow \ i_n.
\end{equation}
Let 
$\Xi_i^\tau$ be the set of features in $\mathbb{N}$ which precede $i$ in the ordering $\tau$, and let
$$
\phi^\tau_i \ \eqdef \ v(\Xi_i^\tau \cup \overline{i}) - v(\Xi_i^\tau)
$$
be $i$'s sequential marginal effect in $\tau$.
Shapley (1953) showed that $\Psi_i [v] = \mathbb{E} [\phi^\tau_i]$ where the expectation is under the uniform distribution on $\Omega$.
Comparing this with Theorem \ref{thm:sq}, the de-weighting on the linked marginal effects (\ref{eq:random_sq_loss}) is the same as that in (\ref{eq:loss_SV_prior}).

\begin{theorem}\label{thm:sq}
Assume $\tau$ has the uniform distribution on $\Omega$. Then
\begin{equation}\label{eq:random_sq_loss}
\lambda_i[v]
=
\mathbb{E} \left [\frac{n-|\Xi_i^\tau|}{n+1} \phi^\tau_i \right ].
\end{equation}
\end{theorem}

\begin{proof}
See Supplemental Appendix A3. 
\end{proof}

\begin{figure}[htb]
\centering
\begin{minipage}{.80\linewidth}
\begin{algorithm}[H]\label{alg:loss_random_sq}
\SetAlgoLined
Pick a large integer $\gamma \ge 100$ as the sample size for $\tau$\;
$\hat \lambda_i \longleftarrow 0$ \ and \ $\hat s_i \longleftarrow 0$\;
\For{$k \gets 1$ to $\gamma$}{
	Randomly select a $\tau$ from $\Omega$\;
	$\hat \lambda_i \longleftarrow \hat \lambda_i + \frac{n-|\Xi_i^\tau|}{n+1} \phi^\tau_i$ \ and \
	$\hat s_i \longleftarrow \hat s_i + \left (\frac{n-|\Xi_i^\tau|}{n+1} \phi^\tau_i \right )^2$\;
}
$\hat \lambda_i \longleftarrow \hat \lambda_i/\gamma$ \ and \ $\hat s_i \longleftarrow \hat s_i/\gamma$\;
$\hat \sigma_i \longleftarrow \sqrt{(\hat s_i)^2 - (\hat \lambda_i)^2} / \sqrt{\gamma}$ \;
return $z \longleftarrow \hat \lambda_i / \hat \sigma_i $. 
\vskip .3cm
\caption{Estimate $\lambda_i[v]$ by $\gamma$ random orderings}
\end{algorithm}
\end{minipage}
\end{figure}

Theorem \ref{thm:sq} signals a hypothesis test for zero mean of the weighted sequential marginal effects $\frac{n-|\Xi_i^\tau|}{n+1} \phi^\tau_i$,
$$
H_0: \lambda_i[v]=0 \qquad \mathrm{versus} \qquad H_1: \lambda_i[v]\not = 0.
$$
The experimental designer randomizes a large number of $\tau$'s in (\ref{eq:random_sq_loss}) to generate a random i.i.d. sample of $\frac{n-|\Xi_i^\tau|}{n+1} \phi^\tau_i$.
From this sample, we can extract many useful statistics, such as the sample median, mean, standard deviation, confidence intervals, and quantiles.
The sample mean  $\hat \lambda_i$ and its standard deviation $\hat\sigma_i$ help to decide if $\lambda_i[v]$ is statistically zero or not; the other statistics are useful in choosing the sample size for $\tau$.
Algorithm \ref{alg:loss_random_sq} details these steps, where $\hat s_i$ is for the expected value of squared $\frac{n-|\Xi_i^\tau|}{n+1} \phi^\tau_i$.
Under the null hypothesis $H_0$, the z-statistic $\hat \lambda_i / \hat \sigma_i$ has an approximate standard normal distribution.
At the $\alpha\%$ significance, therefore, we identify $X_i$ as an irrelevant feature if the test statistic $\left |z \right | < z_{_{1-\alpha/2}}$, where $z_{_{1-\alpha/2}}$ is the critical value.
Otherwise, $X_i$ is classified as relevant.
In the algorithm, all $\lambda_i[v]$'s and their standard deviations are estimated at the same time. 

Candidate features behave differently in the ordering (\ref{fig:sequence}).
For an irrelevant one, if it lies in the first few positions in $\tau$, its marginal effect could be large due to the sampling errors.
Omitted-variable bias could also account for the effect if the first few features in $\tau$ does not completely contain $\mathbf{S}$.
As a consequence of the sampling errors, the effect may even have different signs in different orderings.
But with the same probability, the feature could be in the last few positions in $\tau$, resulting in almost zero marginal effect.
Thus, the standard error estimate would be large while the mean value remains low.
Therefore, its z-statistic tends to be insignificant.
In contrast, a relevant variable $X_i$ would have nontrivial marginal effects in all positions for any $\tau$ and the effect likely remains the same sign.
On average, the effect's magnitude tends to gradually decrease with the feature's increasing position in $\tau$ because its explanatory and predictive power is partially mitigated by the features in $\Xi_i^\tau$.
So, its mean marginal effect remains high and its standard deviation remains low. 
Therefore, its z-statistic is high in magnitude, and it is likely to be selected into $\hat{\mathbf{S}}$.

\begin{figure}[htb]
\centering
\begin{minipage}{.83\linewidth}
\begin{algorithm}[H]\label{alg:sequential_acceptance}
\SetAlgoLined
$\mathbb{R} \longleftarrow \mathbb{N}$ \ and \ $\hat{\mathbf{S}} \longleftarrow \emptyset$\;
\While{$\mathbb{R} \not = \emptyset$}{
Apply Algorithm \ref{alg:loss_random_sq} to estimate $\lambda_i [v_{_\mathbb{R}}]$ and its standard error for all $i\in \mathbb{R}$, while keeping the features in $\hat{\mathbf{S}}$ as additional regressors in the regression\;
Decompose $\mathbb{R} = \mathbb{W} \cup (\mathbb{R}\setminus \mathbb{W})$ where $\mathbb{W}$ contains $\mathbb{R}$'s features whose absolute z-statistics in $v_{_\mathbb{R}}$ are larger than $z_{_{1-\alpha/2}}$\;
	\eIf{ $\mathbb{W} \not= \emptyset$}{
		$\mathbb{R} \longleftarrow \mathbb{R} \setminus \mathbb{W}$ \ and \
		$\hat{\mathbf{S}} \longleftarrow \hat{\mathbf{S}} \cup \mathbb{W}$\;
	}{
		Return $\hat{\mathbf{S}}$.
	}
}
\vskip .3cm
\caption{$v$-adaptive sequential acceptance}
\end{algorithm}
\end{minipage}
\end{figure}

Finally, the density (\ref{eq:SV-prior}) may be too generic and neglects the specific structure of the function $v$.
Algorithm \ref{alg:sequential_acceptance} provides a $v$-adaptive solution by sequentially accepting features into $\hat{\mathbf{S}}$.
At the beginning, all candidates have the same likelihood to be selected.
After the first group is added to $\hat{\mathbf{S}}$, the rest features have the same probability to be adopted by $\hat{\mathbf{S}}$.
The decision depends upon previously accepted features.
So a partially redundant feature would likely be rejected in the second and following rounds, if its counterpart is already in $\hat{\mathbf{S}}$. 
After the second group is also added to $\hat{\mathbf{S}}$, the remaining features' performance is conditional on the updated $\hat{\mathbf{S}}$.
This fashion continues until no feature could be added to $\hat{\mathbf{S}}$.
From the second-round acceptance, the conditional probability density $P_{_T}$ in $\mathbb{R}$ has bound with the function $v$, through previously accepted features.

\section{Compare with Other Methods in Simulation Studies} \label{sect:compare_simulation}

In this section, we conduct several simulation experiments to study the performance of our new selection method.\footnote{EViews, MatLab, and R programs are available upon request.} 
The performance is compared with those from other methods widely used in practice, including AIC, BIC, Stepwise regression, Lasso, and adaptive Lasso.
Also, we investigate the accuracy's robustness to different choices of $v$. 
In terms of exact identification of the actual models, the new approach significantly outperforms the others.
Also, it is robust to the following choices of $v$.

There are five choices for $v(T)$ in the experiments. 
First, we use the R-squared which is the percentage of the dependent variable's variance explained by the model $T$.
The second choice is the adjusted R-squared, adjusted by the model size $|\mathbf{S}|$.
The F-statistic is the next option.
The fourth is the BIC statistic which is adjusted by the number of regressors and the number of data observations.
Last, we use the root mean squared error (RMSE) of out-of-sample forecast.
In this option, the data are randomly split into 2 parts: 80\% are used for estimation and 20\% for out-of-sample forecast.
We use $\lambda^{R2}, \lambda^{AR2}, \lambda^{F}, \lambda^{BIC}, \lambda^{RM}$ to indicate these five options in Table \ref{tb:compareselectionmethods}, respectively.

In the experiments, we let $n=20$ and $|\mathbf{S}|=2,4,6,...,18$.
Given a specific value of $|\mathbf{S}|$, the real relationship is
\begin{equation} \label{eq:simulate_data}
Y_t = \beta_0 + \sum\limits_{i=1}^{|\mathbf{S}|} \beta_i X_{it} + \sum\limits_{i=|\mathbf{S}|+1}^{20} 0 \times X_{it} + \epsilon_t
\end{equation}
for some model-specific unknown coefficients $\beta_i$ and white noise $\epsilon_t$. 
Thus, the true model variables are $\mathbf{S} = \{ 1, 2, \cdots, |\mathbf{S}| \}$.
Each model has a dataset of 100 simulated observations. 
In generating a dataset, we first simulate 20 independent variables using the normal random number generator. 
To add interdependence among them, we multiply these $20\times 1$ vectors by a $20\times 20$ matrix, which is also randomly generated. 
To remove the normality, we transform the data by a nonlinear function, such as exponential, square, cubic, or logarithmic of absolute value. 
We then simulate a set of coefficients $\beta_i$ and the residuals $\epsilon_t$. 
Finally, we apply (\ref{eq:simulate_data}) to calculate the dependent variable $Y_t$.

We use the following options: $100$ random orderings and $.05$ significance level in Algorithm \ref{alg:sequential_acceptance};
for the AIC and BIC approaches, search all $2^n$ subsets of $\mathbb{N}$ and choose the subset with the least information criterion as the estimated model $\hat{\mathbf{S}}$;
for the Lasso or adaptive Lasso method, apply the tenfold cross-validation to calibrate the penalty coefficient. 
There are multiple implementations of Lasso, adaptive Lasso, and stepwise regression by the software packages Eviews, MatLab, and R; only the best results are reported in Table \ref{tb:compareselectionmethods}.
For each simulated dataset, we apply any of these variable selection methods to find an estimated model $\hat{\mathbf{S}}$.
After that, we compute the discrepancy between  $\hat{\mathbf{S}}$ and  $\mathbf{S}$,
including the numbers of two types of errors in $\hat{\mathbf{S}}$.
Finally, we aggregate the discrepancy statistics for all 1,000 datasets.

\begin{table}
\caption{Comparison of Selection Methods in 1,000 Simulated Models$^*$}
\label{tb:compareselectionmethods}
\centering
{ \small
\begin{tabular}{l|c|||c|c|c|c|c||c|c|c|c|c} \hline \hline
$|\mathbf{S}|$&Discrepancy&Lasso&aLasso&Stepwise&AIC&BIC&$\lambda^{R2}$&$\lambda^{AR2}$&$\lambda^{F}$&$\lambda^{BIC}$&$\lambda^{RM}$\\ \hline \hline
\parbox[t]{2mm}{\multirow{5}{*}{2}} 
&$\hat{\mathbf{S}}=\mathbf{S}$              	&159&159&415&52 &458&935&943&934&942&918\\ 
&$|\mathbf{S}\setminus\hat{\mathbf{S}}|=1$		&6  &7  &10 &3  &12 &4  &7  &9  &6  &12 \\
&$|\mathbf{S}\setminus\hat{\mathbf{S}}|\ge 2$   &0  &0  &0  &0  &1  &0  &0  &0  &0  &4  \\ 
&$|\hat{\mathbf{S}}\setminus \mathbf{S}|=1$     &29 &27 &367&147&360&63 &55 &67 &59 &71 \\
&$|\hat{\mathbf{S}}\setminus \mathbf{S}|\ge 2$  &811&814&213&805&173&0  &0  &0  &0  &9  \\ \hline
\parbox[t]{2mm}{\multirow{5}{*}{4}} 
&$\hat{\mathbf{S}}=\mathbf{S}$              	&131&126&453&57 &508&935&942&934&942&914\\ 
&$|\mathbf{S}\setminus\hat{\mathbf{S}}|=1$		&8  &7  &7  &2  &7  &5  &9  &9  &18 &11 \\
&$|\mathbf{S}\setminus\hat{\mathbf{S}}|\ge 2$   &0  &0  &0  &1  &0  &1  &1  &1  &2  &12 \\ 
&$|\hat{\mathbf{S}}\setminus \mathbf{S}|=1$     &45 &23 &346&144&341&62 &49 &61 &41 &48 \\
&$|\hat{\mathbf{S}}\setminus \mathbf{S}|\ge 2$  &821&856&198&805&145&4  &1  &1  &3  &16 \\ \hline
\parbox[t]{2mm}{\multirow{5}{*}{6}} 
&$\hat{\mathbf{S}}=\mathbf{S}$              	&55 &63 &510&73 &527&925&940&933&941&913\\ 
&$|\mathbf{S}\setminus\hat{\mathbf{S}}|=1$		&16 &13 &5  &2  &4  &9  &17 &21 &25 &19 \\
&$|\mathbf{S}\setminus\hat{\mathbf{S}}|\ge 2$   &0  &0  &1  &0  &0  &3  &4  &2  &3  &14 \\ 
&$|\hat{\mathbf{S}}\setminus \mathbf{S}|=1$     &53 & 49&313&134&338&37 &44 &59 &37 &42 \\
&$|\hat{\mathbf{S}}\setminus \mathbf{S}|\ge 2$  &887&878&181&796&135&39 &5  &5  &8  &17 \\ \hline
\parbox[t]{2mm}{\multirow{5}{*}{8}}
&$\hat{\mathbf{S}}=\mathbf{S}$              	&20 &28 &550&127&572&923&935&931&940&904\\ 
&$|\mathbf{S}\setminus\hat{\mathbf{S}}|=1$		&19 &18 &16 &8  &10 &22 &27 &30 &26 &27 \\
&$|\mathbf{S}\setminus\hat{\mathbf{S}}|\ge 2$   &1  &5  &2  &2  &2  &6  &5  &6  &4  &15 \\ 
&$|\hat{\mathbf{S}}\setminus \mathbf{S}|=1$     &59 &43 &301&219&312&39 &38 &43 &27 &43 \\
&$|\hat{\mathbf{S}}\setminus \mathbf{S}|\ge 2$  &904&922&137&665&121&21 &16 &8  &7  &19 \\ \hline
\parbox[t]{2mm}{\multirow{5}{*}{10}}
&$\hat{\mathbf{S}}=\mathbf{S}$              	&14 &14 &604&146&607&922&934&931&939&901\\ 
&$|\mathbf{S}\setminus\hat{\mathbf{S}}|=1$		&24 &26 &22 &19 &20 &34 &33 &32 &34 &39 \\
&$|\mathbf{S}\setminus\hat{\mathbf{S}}|\ge 2$   &0  &12 &4  &2  &1  &7  &11 &15 &7  &12 \\ 
&$|\hat{\mathbf{S}}\setminus \mathbf{S}|=1$     &63 &43 &274&327&291&35 &31 &38 &24 &38 \\
&$|\hat{\mathbf{S}}\setminus \mathbf{S}|\ge 2$  &916&948&98 &519&88 &9  &17 &9  &13 &17 \\ \hline
\parbox[t]{2mm}{\multirow{5}{*}{12}}
& $\hat{\mathbf{S}}=\mathbf{S}$				    &12 &  7&660&181&659&924&933&932&937&904\\
&$|\mathbf{S}\setminus\hat{\mathbf{S}}|=1$      &24 & 27&15 &14 &15 &44 &36 &39 &36 &30 \\
&$|\mathbf{S}\setminus\hat{\mathbf{S}}|\ge 2$   &0  & 12&1  &0  &1  &12 &12 &6  &3  &16 \\
&$|\hat{\mathbf{S}}\setminus \mathbf{S}|=1$	 	&66 & 43&251&319&265&19 &29 &25 &19 &37 \\
&$|\hat{\mathbf{S}}\setminus \mathbf{S}|\ge 2$  &918&949&78 &498&65 &7  &10 &13 &11 &19 \\ \hline
\parbox[t]{2mm}{\multirow{5}{*}{14}}
& $\hat{\mathbf{S}}=\mathbf{S}$				    &0  &0  &711&224&675&930&936&930&938&911\\
&$|\mathbf{S}\setminus\hat{\mathbf{S}}|=1$      &24 &29 &11 &5  &12 &45 &45 &43 &47 &46 \\
&$|\mathbf{S}\setminus\hat{\mathbf{S}}|\ge 2$   &0  &10 &1  &0  &0  &14 &8  &3  &1  &4  \\
&$|\hat{\mathbf{S}}\setminus \mathbf{S}|=1$	 	&72 &59 &224&303&253&12 &21 &21 &18 &27 \\
&$|\hat{\mathbf{S}}\setminus \mathbf{S}|\ge 2$  &907&905&71 &477&63 &2  &9  &5  &12 &15 \\ \hline	
\parbox[t]{2mm}{\multirow{5}{*}{16}}
& $\hat{\mathbf{S}}=\mathbf{S}$				    &1  &  2&781&288&766&931&937&929&938&919\\
&$|\mathbf{S}\setminus\hat{\mathbf{S}}|=1$      &137&92 &7  &6  &7  &57 &58 &54 &51 &52 \\
&$|\mathbf{S}\setminus\hat{\mathbf{S}}|\ge 2$   &4  &23 &0  &0  &0  &3  &1  &1  &0  &1  \\
&$|\hat{\mathbf{S}}\setminus \mathbf{S}|=1$	 	&359&322&187&291&184&12 &7  &17 &11 &21 \\
&$|\hat{\mathbf{S}}\setminus \mathbf{S}|\ge 2$  &502&567&31 &421&47 &0  &2  &1  &3  &14  \\ \hline
\parbox[t]{2mm}{\multirow{5}{*}{18}}
& $\hat{\mathbf{S}}=\mathbf{S}$				    &0  &  0&878&321&859&934&941&931&940&920\\
&$|\mathbf{S}\setminus\hat{\mathbf{S}}|=1$      &297&325&6  &6  &3  &59 &62 &69 &61 &69 \\
&$|\mathbf{S}\setminus\hat{\mathbf{S}}|\ge 2$   &64 &55 &0  &0  &0  &1  &0  &2  &0  &1  \\
&$|\hat{\mathbf{S}}\setminus \mathbf{S}|=1$	 	&622&599&121&654&143&8  &2  &6  &3  &13 \\
&$|\hat{\mathbf{S}}\setminus \mathbf{S}|\ge 2$  &21 &23 &3  &23 &2  &0  &0  &0  &0  &9  \\ \hline\hline	
\end{tabular} 
}
\end{table}

Table \ref{tb:compareselectionmethods} summarize the discrepancy statistics for all these six selection methods and five choices of $v$ in the 1,000 simulated models for nine values of $|\mathbf{S}|$.
We compare these results from the aspects of precision, over-fitting, under-fitting, and computational cost.
\begin{itemize}
\item Exact Identification $\hat{\mathbf{S}}=\mathbf{S}$:
the $\lambda$-type solutions have an average accuracy of 92\%, compared to 34\% for the others. 
When either $|\mathbf{S}\setminus\hat{\mathbf{S}}|=1$ or $|\hat{\mathbf{S}}\setminus \mathbf{S}|=1$, the mismatch has only one or two features.
In the near-matched cases, the $\lambda$ methods also perform extremely well.
\item Under-Fitting $\mathbf{S}\setminus \hat{\mathbf{S}} \not = \emptyset$:
based on these numbers in the table, both the $\lambda$ methods and the others have similar chances of under-fitting.
\item Over-Fitting $\hat{\mathbf{S}}\setminus \mathbf{S} \not = \emptyset$:
the $\lambda$ methods successfully block the irrelevant features from being selected. 
The average chance to make this type of error is about 4\%, compared to $54\%$ for the others.
\item Computational Cost: both the AIC and BIC methods are extremely costly in computation.
The programming experience also shows that the Lasso and adaptive Lasso have the same magnitude of computational cost as the $\lambda$ methods.
The fastest is the stepwise regression.
\end{itemize}

Finally, the last five columns show that the performance of the $\lambda$ methods are not sensitive to the choices of $v$.
Also, the accuracy is robust to the sizes of $\mathbf{S}$.
However, the size-adjusted methods ($\lambda^{AR2}$ and $\lambda^{BIC}$) are slightly better than the un-adjusted ones ($\lambda^{R2}$, $\lambda^{F}$, and $\lambda^{RM}$).
For $\lambda^{R2}$ or $\lambda^{F}$, frankly, $\lambda_i[v]$ has a small positive but unknown value for an irrelevant feature $X_i$.
This could be an issue when the standard error is neither large nor small.
The $\lambda^{RM}$ looks slightly less accurate than the other four, but it uses only 80\% of data for estimation.

\section{Conclusions}\label{sect:conclusion}
 
In this paper, we provide a game-theoretic mechanism design to study a fundamental issue in statistical learning.
When observing the performance $v$ over all subsets of the candidate features, we set up a coalitional game.
However, we cannot simply apply the Shapley value to the game as the efficiency axiom is invalid.
By modifying the axiom to fairly distribute the expected payoff of the potential model, the design sets its objective to lay all the features in a Procrustean bed, allowing neither expected under-fitting nor expected over-fitting.
Irrelevant features behave oddly under the procrustean rule; their marginal effects in the game oscillate around a negligible level, like that of a dummy player.
The rule is set, however, at the aggregate level, not at the individual level, because individual features are not \textit{a priori} known as relevant or irrelevant.
By estimating the unique solution to the objective, we convert the feature selection into a statistical decision problem of hypothesis testing.

A few advantages could explain the $\lambda$ method's superior performance over the traditional methods, as demonstrated in the simulation studies. 
First, by exactly matching the individual contributions with the expected collective performance, we effectively invalidate any discrepancy measurement which other methods attempt to minimize.
Secondly, the formula and estimation procedure are identical for any choice of the payoff function $v$.
Also, as illustrated in the simulation studies, the results slightly vary with the choices of $v$ while the other methods largely depend on their model fit statistics.
Lastly, the $\lambda$ approach has no exchange between the model fit and model size where the fit statistic is subject to multiplicity and vulnerable to other assumptions (e.g, normality and independence). 
The penalty, if any, is replaced with an exact match.

One could extend the solution from different angles.
First, in a real data analysis, the payoff function $v$ generally contains heterogeneous uncertainty. 
In the last section, for example, all five $v(T)$'s are subject to the sampling errors, varying with $T$.
Thus, the solution (\ref{eq:loss_SV_prior}) is not deterministic and the $\hat \sigma_i$ from Algorithm \ref{alg:loss_random_sq} could underestimate the standard deviation of $\lambda_i[v]$.
One remedy is to choose a collective significance statistic for $v(T)$ or an individual significance statistic for its sequential increment $\phi^\tau_i$.
For example, one could use the F-statistic for $v(T)$, and $X_i$'s absolute t-statistic or likelihood ratio score as $\phi^\tau_i$.
The remedy makes the decision based on $\hat \lambda_i$ alone and ignores its standard error.
Secondly, the functional equation (\ref{eq:divide_expectation_lambda}) only matches the expected values. 
The mean value captures one type of central tendency but ignores the other aspects of the complete profile of $v$, such as the shape, the likelihood, and the diminishing marginality.
Next, apparently, not all functions of $v$ work well in Algorithms \ref{alg:loss_random_sq} and \ref{alg:sequential_acceptance}, for example, when $v(T)=0$ for all $T$'s. 
One question is what $v$'s ensure that the solution enjoys some good properties, such as Oracle and efficiency.
Of course, one or more alternative $v$'s could be used for cross-validation or robustness check.
Finally, in Algorithms \ref{alg:loss_random_sq} and \ref{alg:sequential_acceptance}, the same $v(T)$ may be computed multiple times for a given $T$. 
A technical challenge is how to avoid the multiplicity.

%

%

\clearpage
\newpage
\setcounter{page}{1}
\section*{Supplemental Appendix: The Proofs\footnote{Supplement to ``Feature Selection by a Mechanism Design" by Xingwei Hu}}

\subsection*{A1. Proof of Theorem \ref{thm:divide_expectation_lambda}}\label{prf:loss_division}
\noindent
There are $\left ( \begin{array}{c} n \\ t \end{array} \right ) = \frac{n!}{t!(n-t)!}$ subsets of size $t$.
By the NID assumption, we have $P_{_T}=\frac{t! (n-t)!}{n!} \delta_t$ and
$$
\mathbb{E} v(\mathbf{S}) - v(\emptyset)
= (\delta_0-1) v(\emptyset) + \sum\limits_{T\subseteq \mathbb{N}: T\not = \emptyset}  \frac{t! (n-t)!}{n!} \delta_t v(T).
$$

\noindent
We also represent the sum of the loss functions (\ref{eq:lambda}) in terms of $v(T)$,
$$
\begin{array}{rcl}
\sum\limits_{i\in \mathbb{N}} \lambda_i[v]
&=&
\sum\limits_{i\in \mathbb{N}} \ \sum\limits_{T \subseteq \mathbb{N}: i\not \in T} 
P_{_T} \left [ v(T \cup \overline{i}) - v(T) \right ]  \\

&=&
\sum\limits_{i\in \mathbb{N}} \ \sum\limits_{T \subseteq \mathbb{N}: i\not \in T} 
P_{_T} v(T \cup \overline{i}) 
- 
\sum\limits_{i\in \mathbb{N}} \ \sum\limits_{T \subseteq \mathbb{N}: i\not \in T} 
P_{_T} v(T) \\

&\stackrel{Z=T \cup \overline{i}}{=}&
\sum\limits_{i\in \mathbb{N}} \ \sum\limits_{Z \subseteq \mathbb{N}: i\in Z}
P_{_{Z\setminus \overline{i}}} v(Z)
- 
\sum\limits_{T \subseteq \mathbb{N}} v(T)
\sum\limits_{i \in \mathbb{N} \setminus T} P_{_T}  \\

&\stackrel{T=Z}{=}&
\sum\limits_{T \subseteq \mathbb{N}: T\not=\emptyset} v(T) 
\sum\limits_{i \in T} P_{_{T\setminus \overline{i}}}
- 
\sum\limits_{T \subseteq \mathbb{N}} 
\frac{(n-t)t!(n-t)!\delta_t}{n!}
v(T)  \\

&=&
\sum\limits_{T \subseteq \mathbb{N}: T\not=\emptyset} 
\frac{t (t-1)! (n-t+1)!\delta_{t-1}}{n!} v(T) 
- \sum\limits_{T \subseteq \mathbb{N}} 
\frac{(n-t) t! (n-t)!\delta_t}{n!} v(T) \\

&=&
- n\delta_0 v(\emptyset) + \sum\limits_{T \subseteq \mathbb{N}: T\not=\emptyset} 
\frac{t! (n-t)! \left [(n-t+1)\delta_{t-1}-(n-t)\delta_t \right ]}{n!}
v(T) 
\end{array}
$$
\noindent
For (\ref{eq:divide_expectation_lambda}) to hold for any $v$, the coefficients of $v(T)$ match on both sides of
(\ref{eq:divide_expectation_lambda}):
$$
\left \{
\begin{array}{rcl}

-n\delta_0
&=&
\delta_0 - 1, \\

(n-t+1)\delta_{t-1}-(n-t)\delta_t
&=&
\delta_t, \hspace{1.5cm} \forall \ 1\le t \le n.
\end{array}
\right .
$$
The only solution to the system is $\delta_0 = \delta_1 = \cdots = \delta_n = \frac{1}{n+1}$, i.e.,
$P_{_T} = \frac{t!(n-t)!}{(n+1)!}$.

\subsection*{A2. Proof of Theorem \ref{thm:expected_SV}} \label{prf:expected-SV}
\noindent
If $i\not \in Z \subseteq \mathbb{N}$, then  $i$ is a dummy player in $v_{_Z}$ and $\Psi_i[v_{_Z}]=0$.
When $i \in Z$, by (\ref{eq:shapley_value}) and the shorthand $t=|T|$, its Shapley value in $v_{_Z}$ is
$$
\begin{array}{rcl}
\Psi_i[v_{_Z}]
&=&
\sum\limits_{W\subseteq \mathbb{N}: i\in Z\setminus W}  \frac{(|W|)!(n-|W|-1)!}{n!} \left [ v_{_Z}(W\cup \overline{i}) - v_{_Z} (W) \right ] \\
&=&
\sum\limits_{W\subseteq \mathbb{N}: i\in Z\setminus W}  \frac{(|W|)!(n-|W|-1)!}{n!} \left [ v(Z\cap (W \cup \overline{i})) - v (Z\cap W) \right ]\\
&\stackrel{T=Z\cap W}{=}&
\sum\limits_{T\subseteq Z: i\in Z\setminus T} \frac{v(T\cup \overline{i}) - v(T)}{n!} \sum\limits_{W\subseteq \mathbb{N}:W \cap Z = T}  (|W|)!(n-|W|-1)! \\
&\stackrel{U=W\setminus T}{=}&
\sum\limits_{T\subseteq Z: i\in Z\setminus T} \frac{v(T\cup \overline{i}) - v(T)}{n!} \sum\limits_{U\subseteq \mathbb{N}\setminus Z}  (t+|U|)!(n-t-|U|-1)! \\
&=&
\sum\limits_{T\subseteq Z: i\in Z\setminus T} \frac{v(T\cup \overline{i}) - v(T)}{n!} \sum\limits_{u=0}^{n-|Z|}\sum\limits_{U\subseteq \mathbb{N}\setminus Z: |U|=u} (t+u)!(n-t-u-1)! \\
&=&
\sum\limits_{T\subseteq Z: i\in Z\setminus T} \frac{v(T\cup \overline{i}) - v(T)}{n!}\ \sum\limits_{u=0}^{n-|Z|}\ (t+u)!(n-t-u-1)! 
	\left ( \begin{array}{c} n-|Z| \\ u \end{array} \right ) \\
&=&
\sum\limits_{T\subseteq Z: i\in Z\setminus T} \frac{(n-|Z|)! t! (|Z|-t-1)!}{n!}\left [v(T\cup \overline{i})-v(T) \right ] \\
&&
\sum\limits_{u=0}^{n-|Z|}\ \left ( \begin{array}{c} t+u \\ u \end{array} \right ) \left ( \begin{array}{c} n-t-u-1 \\ n-|Z| - u\end{array}   \right )\\
&=&
\sum\limits_{T\subseteq Z: i\in Z\setminus T} \frac{(n-|Z|)! t! (|Z|-t-1)!}{n!}\left [v(T\cup \overline{i})-v(T) \right ] 
	\left ( \begin{array}{c} n \\ n-|Z| \end{array}   \right )  \\
&=&
\sum\limits_{T\subseteq Z: i\in Z\setminus T} \frac{t!(|Z|-t-1)!}{(|Z|)!} \left [ v(T\cup \overline{i}) - v(T)\right ]
\end{array}
$$
where Lemma 1 of Hu (2006) simplifies the last sum in the seventh equality.
Given the NID assumption for $\mathbf{S}$, we apply the above identities of $\Psi_i[v_{_Z}]$ to get
$$
\begin{array}{rcl}
\mathbb{E} \Psi_i [v_{_\mathbf{S}}] 
&=&
\sum\limits_{Z\subseteq \mathbb{N}} P_{_Z} \Psi_i[v_{_Z}] 
=
\sum\limits_{Z\subseteq \mathbb{N}: i \in Z} P_{_Z} \Psi_i[v_{_Z}] \\
&=&
\sum\limits_{Z\subseteq \mathbb{N}: i\in Z} \frac{(|Z|)!(n-|Z|)!}{n!}\delta_{_{|Z|}} \sum\limits_{T\subseteq Z: i\in Z\setminus T} \frac{t!(|Z|-t-1)!}{(|Z|)!} \left [ v(T\cup \overline{i}) - v(T)\right ] \\
&=&
\sum\limits_{T\subseteq \mathbb{N}: i \not \in T} \frac{t!}{n!} \left [v(T\cup \overline{i}) - v(T)\right ]\ \sum\limits_{Z\subseteq \mathbb{N}:i\in Z, Z \supseteq T} (n-|Z|)! (|Z|-t-1)! \delta_{_{|Z|}}\\
&=&
\sum\limits_{T\subseteq \mathbb{N}: i \not\in T} \frac{t!}{n!} \left [v(T\cup \overline{i}) - v(T)\right ] \\
&&
\sum\limits_{z=t+1}^n \ \sum\limits_{Z\subseteq \mathbb{N}: |Z|=z, Z \supseteq T, i\in Z, i\not \in T} (n-z)! (z-t-1)! \delta_z \\
&=&
\sum\limits_{T\subseteq \mathbb{N}: i \not\in T} \frac{t!}{n!} \left [v(T\cup \overline{i}) - v(T)\right ] \sum\limits_{z=t+1}^n  (n-z)! (z-t-1)! \delta_z
\left ( 
\begin{array}{c}
n-t-1 \\
z-t-1
\end{array}  
\right )
\\

&=&
\sum\limits_{T\subseteq \mathbb{N}: i \not\in T} \frac{t!(n-t-1)!}{n!} \left [ v(T\cup \overline{i}) - v(T)\right ] \sum\limits_{z=t+1}^n \delta_z.  
\end{array}
$$
Therefore, comparing $\mathbb{E} \Psi_i [v_{_\mathbf{S}}]$ with $\lambda_i [v]$ in (\ref{eq:loss_SV_prior}), we have
$$
\sum\limits_{z=t+1}^n \delta_z = \frac{n-t}{n+1}, \quad \forall \ t = 0, 1, ..., n-1,
$$
which has the only solution $\delta_0 = \delta_1 = \cdots =\delta_n = \frac{1}{n+1}$, i.e., $P_{_T} = \frac{t!(n-t)!}{(n+1)!}$.

\subsection*{A3. Proof of  Theorem \ref{thm:sq}}\label{prf:thm-sq}

\noindent
As there are $n!$ orderings in $\Omega$, each ordering $\tau$ occurs with probability $\frac{1}{n!}$. 
Also, there are $(|\Xi_i^\tau|)!$ permutations in $\Xi_i^\tau$ and $(n-1-|\Xi_i^\tau|)!$ in $\mathbb{N}\setminus \Xi_i^\tau \setminus \overline{i}$, the set of players preceded by $i$ in $\tau$.
Therefore, the probability of $\Xi_i^\tau = T$ is $\frac{(|\Xi_i^\tau|)!(n-1-|\Xi_i^\tau|)!}{n!} = \frac{t!(n-1-t)!}{n!}$.
Using the law of total expectation, we have
$$
\begin{array}{rcl}
\mathbb{E} \left [ \frac{n-|\Xi_i^\tau|}{n+1} \phi^\tau_i \right ] 
&=& 
\sum\limits_{T\subseteq \mathbb{N}\setminus \overline{i}} \mathrm{Prob} ( \Xi_i^\tau=T) \mathbb{E} \left [ \frac{n-|\Xi_i^\tau|}{n+1} \phi^\tau_i  \  | \ \Xi_i^\tau=T   \right ] \\
&=&  
\sum\limits_{T\subseteq \mathbb{N}\setminus \overline{i}}  \frac{t!(n-1-t)!}{n!} \frac{n-t}{n+1} \left [v(T \cup \overline{i}) - v(T) \right ] \\
&=& 
\sum\limits_{T\subseteq \mathbb{N}\setminus \overline{i}}  \frac{t!(n-t)!}{(n+1)!} \left [v(T \cup \overline{i}) - v(T) \right ] 
=
\lambda_i [v].
\end{array}
$$

\begin{thebibliography}{8}
\bibitem{ClydeGeorge2004}
Clyde, M.,  George, E. I.:
Model uncertainty.
Statist. Science \textbf{19}, 81--94 (2004).

\bibitem{DziakCoffmanLanzaLiJermiin}
Dziak, J.J., Coffman, D.L., Lanza, S.T., Li, R., Jermiin, L.S.:
Sensitivity and specificity of information criteria.
Briefings in Bioinformatics, \textbf{21}(2), 553--565 (2020).

\bibitem{FanLi2006}
Fan, J., Li, R.: 
Statistical Challenges With High Dimensionality: Feature Selection in Knowledge Discovery,
In: Sanz-Sole, M. et al. (eds) Proceedings of the Madrid International Congress of Mathematicians.
European Mathematical Society (2006).

\bibitem{GeorgeMcCulloch1997}
George, E.I.,   McCulloch, R.E.:
Approaches for Bayesian variable selection.
Statistica Sinica  \textbf{7}, 339--373 (1997).

\bibitem{Hu2006}
Hu, X.:
An asymmetric Shapley-Shubik power index.
Intl. J. Game Theory \textbf{34}, 229--240 (2006).

\bibitem{OHaraSillanpaa2009}
 O'Hara, R.B.,  Sillanpaa, M.J.:
A review of Bayesian variable selection methods: what, how and which.
Bayesian Analysis \textbf{4}, 85--118 (2009).

\bibitem{Shapley1953}
Shapley, L.S.:
A value for n-person games.
In: Kuhn, H., Tucker, A. (eds.) Annals of Mathematics Studies, Vol. 28, pp.307--317.
Princeton University Press, Princeton, New Jersey (1953). \doi{10.1515/9781400881970-018}

\bibitem{Tibshirani1996}
Tibshirani, R.:
Regression shrinkage and selection via the LASSO.
J. Roy. Statist. Soc  B \textbf{58}, 267--288 (1996). 

\bibitem{Zou2006}
Zou, H.:
The adaptive Lasso and its Oracle properties.
J. Amer. Stat. Assoc. \textbf{101}(476): 1418-1429 (2006).

\end{thebibliography}
\end {document}